\documentclass[sts]{imsart}

\usepackage{amssymb,amsbsy,amsmath,amscd,amsthm,amsfonts,MnSymbol}
\usepackage{cite}
\usepackage[utf8]{inputenc}
\usepackage{enumerate}
\usepackage{hyperref}
\usepackage{dsfont}
\usepackage{graphicx}
\usepackage{enumitem}
\usepackage{wrapfig}

\usepackage{algorithm}
\usepackage{algorithmic}

\usepackage{enumerate}

\newtheorem{definition}{Definition}
\newtheorem{theorem}{Theorem}

\newtheorem{example}{Example}
\newtheorem{lemma}{Lemma}

\newtheorem{assumption}{Assumption}

\newcommand{\R}{\mathbb R}
\newcommand{\E}{\mathbb E}

\newcommand{\dist}{\mathrm{d_H}}

\newcommand{\LS}{\mathrm{LS}}
\newcommand{\GS}{\mathrm{GS}}
\newcommand{\Sm}{\mathrm{S}}

\newcommand{\Var}{\textrm{Var}}

\newcommand{\DS}{\displaystyle}

\newcommand{\EE}{\mathbb{E}}

\newcommand{\PP}{\mathbb{P}}

\begin{document}

\begin{frontmatter}

\title{Propose, Test, Release: Differentially private estimation with high probability}
\runtitle{DP location estimators}



\author{Marco Avella-Medina \thanks{Columbia University, Department of Statistics, marco.avella@columbia.edu} \and Victor-Emmanuel Brunel \thanks{ENSAE ParisTech, Department of Statistics, victor.emmanuel.brunel@ensae.fr}}

\runauthor{M. Avella-Medina and V.-E. Brunel}

\setattribute{abstractname}{skip} {{\bf Abstract:} }

\begin{abstract}

We derive concentration inequalities for differentially private median and mean estimators building on the ``Propose, test, release'' (PTR) mechanism introduced by \cite{dworkandlei2009}. We introduce a new general version of the PTR mechanism that allows us to derive high probability error bounds for differentially private estimators. Our algorithms provide the first statistical guarantees for differentially private estimation of the median and mean without any boundedness assumptions on the data, and without assuming that the target population parameter lies in some known bounded interval. Our procedures do not rely on any truncation of the data and provide the first sub-Gaussian high probability bounds for differentially private median and mean estimation, for possibly heavy tailed random variables.

\end{abstract}

\begin{keyword}
\kwd{Differential Privacy}
\kwd{Location estimators}
\kwd{Sub-Gaussian deviations}
\kwd{Median}
\kwd{Median of Means estimator}
\end{keyword}

\end{frontmatter}

\section{Introduction}

Differential privacy has emerged as the main approach to privacy  in the theoretical computer science and machine learning literature following the path breaking work of \cite{dworketal2006}. This paradigm provides a rigorous mathematical framework for the study and design of privacy-preserving algorithms. This setting assumes that there is a trusted curator that holds data containing some possibly sensitive records of $n$ individuals. The goal of differential privacy is to simultaneously protect every individual record while releasing global characteristics of the database \cite{dworkandroth2014}. This is achieved by constructing randomized algorithms that release noisy versions of the desired outputs, where the noise level is calibrated to prevent any individual level data to be identifiable by querying the database.

Even though the machine learning community has been very prolific in developing differentially private algorithms for complex settings including multi-armed bandit problems \cite{mishraandthakurta2015, tossouandimitrakakis2016,shariffandsheffet2018},  high-dimensional regression \cite{kiferetal2012,talwaretal2015} and  deep learning \cite{abadietal2016,lecuyeretal2018},  some  fundamental statistical questions are only starting to be understood. For example, the first statistical minimax rates of convergence under differential privacy were recently established in \cite{duchietal2018,caietal2019}. Some earlier work framing differential privacy in traditional statistics terms include  \cite{wassermanandzhou2010, lei2011,smith2011,chaudhuriandhsu2012,karwaandslavkovic2016}.  Recent work  has also sought to develop differential privacy tools for statistical inference and hypothesis testing \cite{gaboardi2016,sheffet2017,avella2019,barrientosetal2019}

In this paper we study the simple statistical problem of location parameter estimation and study the non-asymptotic deviations of differentially private location parameter estimators. More specifically, we consider the problem of constructing median and mean estimators that achieve sub-Gaussian finite sample deviations, even when applied to heavy tailed data.

\subsection{Motivation}
It is well known that given a random iid sample $X_1,\dots,X_n$ of sub-Gaussian random variables with $\EE[X_1]=\mu$ and $\Var[X_1]=\sigma^2$,  the empirical mean $\overline{X}_n=\frac{1}{n}\sum_{i=1}^nX_i$ satisfies with probability at least $1-\tau$
\begin{equation*}
 |\overline{X}_n-\mu|\leq \sqrt{\frac{2\sigma^2\log(2/\tau)}{n}}.
\end{equation*}
The accuracy of the empirical mean estimator expressed in the above deviation inequality is a direct consequence of the sub-Gaussian assumption, and such an error bound is called a \textit{sub-Gaussian deviation for the confidence level $\tau$}. In fact the dependence on $\tau$ in the error worsens significantly when the distribution does not have a moment generating function. In particular, when $X_1$ is only assumed to have two finite moments, the error of the empirical mean cannot get smaller in order than $\sqrt{2\sigma^2/(n\tau)}$ as shown in \cite{catoni2012}.  
For the empirical median on the other hand, one does not even need to assume any finite moments in order to establish similar sub-Gaussian deviations. However, there exist estimators of the mean that do achieve sub-Gaussian deviations even when the data only have two finite moments: For instance, the \textit{median of means} estimator \cite{lerasleandoliveira2011}, which first symmetrizes the distribution by taking empirical means of disjoint subsamples of the data, so the population median of these empirical means is close enough to the population mean, and then aggregates these empirical means by taking their empirical median.

In light of \cite{caietal2019}, one may naturally wonder how differential privacy will affect the deviation bounds discussed above. The statistical minimax rates established in \cite{caietal2019} show that the rates of convergence of differentially private mean estimators are described by two terms. The first one correspond to the usual parametric $1/\sqrt{n}$ convergence, while the second one is driven by the differential privacy tuning parameters and  and is of the order $1/n$. Consequently, for large $n$, differential privacy does not come at the expense of slower statistical convergence rates, at least in expectation. However, a notorious technical difficulty renders the study of non-asymptotic deviations challenging for differentially private mean and median estimators: existing algorithms either require the input data to be bounded or assume that the parameter lies in a known interval; see for example \cite{dworketal2006,nissimetal2007,lei2011,smith2011,bassilyetal2014}. This is clearly unsatisfactory from a theoretical and practical perspective as it rules out common distributions used in statistical modeling such as the normal, gamma and t-distributions, just to name a few. This is particularly disturbing for median estimators since the usual non-private empirical median does not even need the existence of finite moments in order to exhibit sub-Gaussian deviations. 

Moreover, to the best of our knowledge, it is not known whether it is possible to find differentially private estimators of the mean that achieve sub-Gaussian deviations even when the data are heavy tailed.

\subsection{Our contributions}

\begin{itemize}

\item We greatly improve the  ``Propose, test, release" (PTR) mechanism that was introduced in \cite{dworkandlei2009}. At the heart of this approach lies the idea of avoiding the usual worst case scenario noise calibration that is omnipresent in the design of differentially private algorithms. This goal is achieved by first exploring, in a privacy-preserving way, whether the data has a favorable configuration that permits to add only a small noise to the desired output. When the data has a bad configuration, the algorithm halts giving a ``no reply''.  We revisit the original mechanism for which only very weak guarantees were given, such as consistency of the estimators. We derive a new refined version of this mechanism that is more general, provides a tight control of the probability of ``no reply'' and minimizes the noise added to the desired output, thus significantly improving the statistical properties of the resulting differentially private estimators.

\item
We provide the first high probability bounds for differentially private estimation of the median and mean, under weak assumptions, similar to the ones required by their usual non-private counterparts. Namely, our median estimator is shown to exhibit sub-Gaussian deviations under the standard assumption requiring the underlying distribution to have a density that is bounded from below in a neighborhood of the population median. Similarly, we provide a median-of-means estimator that is shown to estimate the population mean with sub-Gaussian errors by assuming only the existence of three moments.

\item To the best of our knowledge, our algorithm for mean estimation is the first differentially private estimator of the mean for which one can prove good statistical guarantees without any boundlessness assumptions on the data and without assuming that the population mean lie in some known bounded interval. In particular, our construction does not rely on any truncation of the data and gives optimal, sub-Gaussian deviations when the data are heavy tailed.

\end{itemize}

\section{Preliminaries}

Before describing our techniques we will need to introduce definitions and basic tools from differential privacy that will serve as building blocks for our procedures.

\subsection{Definitions}

For $x=(x_1,\ldots,x_n)\in\R^n$, we denote by $x_{(1)}, \ldots, x_{(n)}$ the reordered coordinates of $x$ in nondecreasing order, i.e. $\DS \min_{1\leq i\leq n} x_i=x_{(1)}\leq \ldots\leq x_{(n)}=\max_{1\leq i\leq n} x_i$. We let $\ell=\lfloor n/2\rfloor$ and $\hat m(x)=x_{(\ell)}$ be the empirical (left) median of $x$.
For any two vectors $x,x'\in\R^n$, we define their Hamming distance $\dist(x,x')$ as the number of coordinates that differ in $x$ and $x'$: $\DS \dist(x,x')=\#\{i=1,\ldots,n:x_i\neq x_i'\}$, where $\#$ stands for cardinality. For all $x\in\R^n$ and $k\geq 0$, we let $B_{\textsf{H}}(x,k)$ be the set of all vectors that differ from $x$ in at most $k$ coordinates, i.e., $B_{\textsf{H}}(x,k)=\{x'\in\R^n:\dist(x,x')\leq k\}$.

In what follows, we refer as \textit{random function} to any function $\tilde h:\R^n\to\R$ such that for all $x\in\R^n$, $\tilde h(x)$ is a Borelian random variable. In this paper, we will use the hat sign to denote non-randomized estimators, and the tilde sign to denote their randomized version. For instance, $\hat h(x)$ would be a nonrandom estimator, i.e., the product of a deterministic algorithm that has input $x\in\R^n$, whereas $\tilde h(x)$ would be a randomized estimator, i.e., the product of a randomized algorithm.

\begin{definition}
	Let $\varepsilon,\delta>0$. A random function $\tilde h$ is called $(\varepsilon,\delta)$-differentially private if and only if for each pair $x,x'\in\R^n$ with $\dist(x,x')\leq 1$ and for all Borel sets $B\subseteq \R$, 
$$\PP[\tilde h(x)\in B]\leq e^\varepsilon \PP[\tilde h(x')\in B]+\delta.$$
\end{definition}

Note that when $\varepsilon=0$, $(0,\delta)$-differential privacy means that the total variation distance between the laws of $\tilde h(x)$ and $\tilde h(x')$ must be bounded by $\delta$, whenever $\dist(x,x')\leq 1$. On the other hand, when $\delta=0$, $(\varepsilon,0)$-differential privacy means that a rescaled version of the total variation distance between the laws of $\tilde h(x)$ and $\tilde h(x')$ must be bounded by $e^\varepsilon-1$, whenever $\dist(x,x')\leq 1$.

The following notions of sensitivity of an output function $h$ are central in the construction of differentially private procedures. In particular, these measures of sensitivity are used in the most basic algorithms that make some output $h(x)$ private by simply releasing instead $h(x)+u$, where $u$ is an independent noise term whose variance is scaled employing these various notions of sensitivity. 

\begin{definition}
	Let $h:\R^n\to\R$ be a given function. 
	\begin{enumerate}
		\item The local sensitivity of $h$ maps any data point $x\in\R^n$ to the (possibly infinite) number $$\LS_h(x)=\sup\{|h(x')-h(x)|: x'\in\R^n, \dist(x,x')\leq 1\}.$$
		 
		\item The global sensitivity of $h$ is the (possibly infinite) number $$\GS_h=\sup_{x\in\R^n}\LS_h(x).$$ 
		\item For all $\beta>0$, the $\beta$-smooth sensitivity of $h$ is the mapping 
	$$\DS \Sm_h^{(\beta)}(x)=\sup_{x'\in\R^n}\left(e^{-\beta \dist(x,x')}LS_{h}(x')\right), \quad x\in\R^n.$$
	\end{enumerate}
\end{definition}
Note that the global sensitivity is a constant number, which does not depend on the point at which the function $h$ is computed.

\begin{example}
\begin{enumerate}
	\item It is easy to see that for all $x\in\R^n$, the local sensitivity of the empirical median is
	$$\LS_{\hat m}(x)=\max\left(x_{(\ell+1)}-x_{(\ell)},x_{(\ell)}-x_{(\ell-1)}\right).$$
	Moreover, for all $\beta>0$ and all $x\in\R^n$, $\DS \Sm_{\hat m}^{(\beta)}(x)=\infty$, and $\GS_{\hat m}=\infty$.
	\item Let $\DS \hat \mu(x)=n^{-1}\sum_{i=1}^n x_i, x\in\R^n$ be the empirical mean function. Then, all the above quantities are infinite.
\end{enumerate}	
\end{example}

In order to enforce the differential privacy of an estimator, usual methods would randomize the estimator by adding some noise to it, that is calibrated by the global or the smooth sensitivity of the estimator. We discuss two such approaches next.

\subsection{Laplace and Gaussian mechanisms}

The Laplace and Gaussian mechanisms are two simple tools used in the differential privacy literature in order to construct private algorithms. The basic idea of these techniques is to make deterministic functions private by adding  random noise calibrated using their sensitivity to the data. Let us review some well known results  for two variants of these constructions of  differentially private estimators.
Recall that the Laplace distribution with parameter $\lambda>0$ is the continuous probability distribution with density $(\lambda/2)e^{-\lambda|u|}, u\in\R$. We denote this distribution by $\mathrm{Lap}(\lambda)$. 

The first part of the following theorem is due to \cite{dworketal2006}. It gives a very simple way to make a function $h$ differentially private. However, it requires the very strong assumption that $h$ has a finite global sensitivity. The second part of the theorem can be found in \cite[Appendix A]{dworkandroth2014}.

\begin{theorem} \label{thm:DP-GS}
	Let $h:\R^n\to\R$ be a function with finite global sensitivity. 
	\begin{enumerate}
	\item
	Let $Z$ be a Laplace random variable with parameter $1$. For all $\varepsilon>0$, the random function $\DS \tilde h(x)=h(x)+\frac{Z}{\varepsilon}\GS_h, x\in\R^n$, is $(\varepsilon,0)$-differentially private.
	\item
	Let $Z$ be a standard normal random variable. For all $\varepsilon>0$, the random function $\DS \tilde h(x)=h(x)+\frac{\sqrt{2\log(1.25/\delta)}Z}{\varepsilon}\GS_h, x\in\R^n$, is $(\varepsilon,\delta)$-differentially private.
	\end{enumerate}
\end{theorem}

The following result is due to \cite{nissimetal2007}, and allows for less restrictive functions $h$.

\begin{theorem} \label{thm:DP-SS}
	Let $h:\R^n\to\R$ and assume that for all $k=1,\ldots,n$ and all $x\in\R^n$, $\LS_h^{(k)}(x)<\infty$. 
	\begin{enumerate}
	\item
	Let $Z$ be a Laplace random variable with parameter $1$. Let $\varepsilon,\delta>0$ and set $\DS \beta=\frac{\varepsilon}{2\log(1/\delta)}$. Then, the random function $$\tilde h(x)=h(x)+\frac{2Z}{\varepsilon}\Sm_h^{(\beta)}(x), \quad x\in\R^n,$$ is $(\varepsilon,\delta)$-differentially private.
	\item
	Let $Z$ be a standard normal random variable. Let $\varepsilon,\delta>0$ and set $\DS \beta=\frac{\varepsilon}{4\{1+\log(2/\delta)\}}$. Then, the random function $$ \tilde h(x)=h(x)+\frac{5\sqrt{2\log(2/\delta)}Z}{\varepsilon}\Sm_h^{(\beta)}(x), \quad x\in\R^n,$$ is $(\varepsilon,\delta)$-differentially private.
	\end{enumerate}
\end{theorem}

Note that these two versions of the Laplace and Gaussian mechanisms, using either the global, or the smooth sensitivities, cannot be used directly for the empirical mean or the empirical median when the data are unbounded, since the two sensitivities are infinite. 
It is important to note that for all $\beta\geq \beta'>0$, all functions $h:\R^n\to\R$ and all $x\in\R^n$, it holds that
\begin{equation} \label{eq:order-sens}
	\LS_h(x)\leq \Sm_h^{(\beta)}(x)\leq \Sm_h^{(\beta')}(x)\leq \GS_h(x)\leq \infty.
\end{equation}
Therefore, a procedure based on the global sensitivity adds more noise the smooth sensitivity, which itself is more aggravating than the local sensitivity. In that sense, using the a procedure based on the local sensitivity would be best. However, standard procedures, such as the Laplace and Gaussian mechanisms described above do not provide privacy guarantees if they calibrate their noise with the local sensitivity, as explained in \cite{nissimetal2007}. This motivates the \textit{propose-test-release} mechanism introduced in the next section, as it uses some sort of local sensitivity, which is less restrictive than the smooth sensitivity. It can be viewed as  a trade-off between the smooth and the local sensitivity noise calibration procedures and hence,  it produces significantly less noisy differentially private estimators.

\section{Propose, test, release}

Here, we describe our main algorithm, that is completely inspired from \cite{dworkandlei2009}. This approach exploits the intuition that while  the local sensitivity does not account for the worst case scenario sensitivity over all data configurations, it could provide a good estimate of the sensitivity of most data sets compatible with standard statistical assumptions.  This insight is theoretically validated by a differentially private procedure that releases a noisy output calibrated according to the local sensitivity if the data configuration was deemed favorable by an initial privacy-preserving test. If the test instead judged the data configuration to be unfavorable, the algorithm stops and gives a ``no reply''. Under appropriate statistical assumptions, we can show that this approach will provide a very accurate response with high probability.

Unlike in \cite{dworkandlei2009}, we do not discretize the parameter space into bins in order to test whether the data are in a favorable configuration. We suggest to instead compute a new notion of fixed-threshold finite sample breakdown point as a means to checking whether the data is amenable to release noisy outputs in a privacy-preserving fashion. This approach has at least three significant advantages: 1) the main idea can be linked to an important concept from robust statistics, namely, the finite sample breakdown point \cite{donohoandhuber1983}. The latter is defined as the minimum fraction of observations that one would need to move arbitrarily in order to get the value of an estimator to diverge, for any finite sample. More precisely, the finite sample breakdown point of a real valued estimator $\hat\theta:\R^n\to\R$ at data points $x\in\R^n$ is defined as $\mbox{BP}(\hat\theta)=n^{-1}\min\{m\geq 0: \sup_{x'\in B_\textsf{H}(x,m)} |\hat\theta(x')-\hat\theta(x)|=\infty\}$. The key computation of our data configuration test is to evaluate an analogous quantity \eqref{hatA} that effectively defines a more refined version of the finite sample breakdown point;  2) our test allows us to  tightly assess  whether the data configuration is favorable. Indeed, the fixed-threshold breakdown point is exactly the quantity that one would like to control in order to assess the sensitivity of  an estimator given a fixed data set. This is to be contrasted to the binning strategy of algorithm $\mathcal M$ in \cite{dworkandlei2009} that tests for errors of size $O(\sigma n^{1/3})$ which entails both a suboptimal probability of ``no reply" and a larger estimation error when the algorithm returns a noisy response ; 3) our mechanism is readily applicable to general  estimators becase the fixed-scale breakdown point can be assessed in an appropriately chosen norm. 

The  PTR paradigm can be combined either with the Laplace, or with the Gaussian mechanism, through the noise that is added to the estimator and the preliminary privacy-preserving data configuration test. The advantage of the Laplace mechanism is its simplicity, but it yields heavier tailed errors, whereas the Gaussian mechanism allows for much lighter noise. For completeness, we describe the two versions of the algorithm, even though we will only use the Gaussian mechanism when applied to the estimation of location parameters.

In order to simplify the notation, we assume that all the data and the parameters of interest are real-valued, but our results can extend to more general spaces.
Let $\hat\theta:\R^n\to\R$ be a fixed function (that would serve as an estimator for some unknown quantity of interest $\theta$, when applied to the data) and let $\eta>0$. Define the function 
\begin{align}
\nonumber \hat A_{\eta} :  \R^n  & \to  \{0,1,2,\ldots,n\} \\
 x & \mapsto  \min\{k\geq 0: \exists x'\in B_\textsf{H}(x,k), |\hat\theta(x')-\hat\theta(x)|>\eta\}. \label{hatA}
\end{align}
Let $Z_1,Z_2$ be two independent random variables and let $a_\delta$ and $b_\delta$ be two positive numbers, to be specified below (according to whether we define the Laplace or the Gaussian version of PTR). Define the randomized functions 
$$\tilde A_{\eta}(x)=\hat A_{\eta}(x)+\frac{a_\delta}{\varepsilon}Z_1$$
and 
$$\tilde \theta_\eta(x)=\begin{cases} \perp \mbox{ if } \tilde A_{\eta}(x)\leq 1+\frac{b_\delta}{\varepsilon} \\ \hat\theta(x)+\frac{\eta}{\varepsilon}a_\delta Z_2 \mbox{ otherwise},\end{cases}$$
for all $x\in\R^n$.

Here, $\perp$ means that the randomized algorithm $\tilde \theta_{\eta}(x)$, when applied to the data $x$, produces a ``no reply''. Intuitively, $\hat A_{\eta}$ can be thought of as the answer to the query ``How many data points should be changed so as to affect the value of $\theta(x)$ by at least $\eta$?". In other words, $\hat A_{\eta}$ quantifies how favorable the configuration of the data is in order to preserve privacy. A small value of $\hat A_{\eta}$ means that the configuration is not favorable, i.e., small changes in the data can affect a lot the output $\hat\theta(x)$ of the deterministic algorithm. In that case, the randomized algorithm will likely prefer to produce no output. From a statistical point of view, the statistic $\hat A_\eta$ is reminiscent of the finite sample breakdown point studied in robust statistics \cite{donohoandhuber1983,huberandronchetti2009}. While the finite sample breakdown point is usually defined as the minimum number of points that needs to be moved arbitrarily before an estimator becomes infinite, $\hat A_\eta$ can be interpreted as a relaxed version of the finite sample breakdown point of the estimator $\hat\theta$ at the threshold $\eta$.

\begin{theorem}[Laplace PTR] \label{thm:DP_PTRLap}
	Let $Z_1$ and $Z_2$ have the Laplace distribution with parameter 1 and let $a_\delta=1$ and $b_\delta=\log(2/\delta)$. The randomized function $\tilde \theta_\eta$ is $(2\varepsilon,\delta)$-differentially private.
\end{theorem}

\begin{theorem}[Gaussian PTR] \label{thm:DP_PTRGauss}
	Let $Z_1$ and $Z_2$ have the standard Gaussian distribution and let $a_\delta=\sqrt{2\log(1.25/\delta)}$ and $b_\delta=2\log(1.25/\delta)$. The randomized function $\tilde \theta_\eta$ is $(2\varepsilon,2e^\varepsilon \delta+\delta^2)$-differentially private.
\end{theorem}

For the Gaussian version, note that if $\varepsilon$ and $\delta$ are smaller than $1$ (which is typically the case), then $\tilde \theta_\eta$ is $(2\varepsilon,(2e+1)\delta)$-differentially private.

The advantage of the standard Laplace mechanism over the Gaussian mechanism is that when the global sensitivity of the estimator is finite, the Laplace mechanism allows for $(\varepsilon,0)$-differentiable privacy, whereas some more slack is unavoidable for the Gaussian mechanism, which only allows for $(\varepsilon,\delta)$-differential privacy for positive $\delta$. In PTR, both Laplace and Gaussian versions yield the additional slack $\delta$ which, from our computations, seems unavoidable in both cases. This is why the Gaussian version seems always preferable to the Laplace one.

Now, these randomized algorithms shall be applied to some data, in order to estimate a quantity $\theta$, such as the population mean, or the population median. The parameter $\eta$ will be chosen in a way that will guarantee that the data is in a favorable configuration with high probability (now, the probability is taken over the randomness of the data). Again, we treat the Laplace and the Gaussian versions of the algorithm separately, for completeness, even though we will only use the Gaussian mechanism in our application.

\begin{theorem}[Laplace version] \label{thm:ErrorPTRLap}
	Let $X=(X_1,\ldots,X_n)$ be i.i.d. real valued data and $\tau\in (0,1)$. Set 
$$\eta^*=\inf\left\{\eta>0:\PP\left[\hat A_\eta(X)\leq 1+\frac{\log(1/(\tau\delta))}{\varepsilon}\right] \leq\tau/2\right\}.$$
Then, for all $\eta\geq\eta^*$, the Laplace version of PTR satisfies, with probability at least $1-2\tau$,
$$|\tilde\theta_{\eta}-\theta|\leq |\hat\theta-\theta|+\frac{\eta}{\varepsilon}\log(1/(2\tau)).$$
\end{theorem}

Note that in this theorem, the probability is computed with respect to the joint randomness of the algorithm and of the data. Furthermore, the second term corresponds to a subexponential type error because the dependence on $\log(1/\tau)$ does not appear inside a square root: This is due to the Laplace mechanism that is embedded in this version of the PTR algorithm. For the Gaussian version, we have the following.

\begin{theorem}[Gaussian version] \label{thm:ErrorPTRGauss}
	Let $X=(X_1,\ldots,X_n)$ be i.i.d. real valued data and $\tau\in (0,1)$. Set 
$$\textstyle \eta^*=\inf\left\{\eta>0:\PP\left[\hat A_\eta(X)\leq 1+\frac{2\log(\frac{1.25}{\delta})+2\sqrt{\log(\frac{2}{\tau})\log(\frac{1.25}{\delta})}}{\varepsilon}\right] \leq \frac{\tau}{2}\right\}.$$
Then, for all $\eta\geq\eta^*$, the Gaussian version of PTR satisfies, with probability at least $1-2\tau$,
$$|\tilde\theta_{\eta}-\theta|\leq |\hat\theta-\theta|+\frac{2\eta}{\varepsilon}\sqrt{\log(2/\tau)\log(1.25/\delta)}.$$
\end{theorem}

The quantity $\eta^*$ may be infinite: In that case, the theorem is vacuous. This is the case, for instance, when $\hat\theta$ is the empirical mean, for data that are not compactly supported. In order to circumvent this issue without truncating the data and assuming that the population mean belongs to some known bounded interval, we apply our PTR approach to estimators based on empirical medians, which are more robust than the empirical mean. Given Theorems \ref{thm:ErrorPTRLap} and \ref{thm:ErrorPTRGauss}, for these estimators, the main challenge is to bound the quantity $\eta^*$.

\section{PTR for location parameters}

In this section, we let $X_1,\ldots,X_n$ be i.i.d. samples drawn from a distribution on the real line, where $n\geq 1$ is a fixed integer. We focus on the estimation of two location parameters: the population median and the population mean. Of course, our results extend easily to other population quantiles.

In all the following, we let $C=1+\frac{2\log(1.25/\delta)+2\sqrt{\log(2/\tau)\log(1.25/\delta)}}{\varepsilon}$ (which enters in the definition of $\eta^*$ for the Gaussian version of the PTR algorithm, see Theorem \ref{thm:ErrorPTRGauss}).

\subsection{Median estimation}

We will only require the following distributional assumption in the derivation of our deviation inequalities for our  private median estimators.
\begin{assumption} \label{Ass:1}
	The distribution of $X_1$ has a density $f$ with respect to the Lebesgue measure and it has a unique median $m$. Moreover, there exist positive constants $r,L$ such that $f(u)\geq L$, for all $u\in [m-r,m+r]$.
\end{assumption}
In particular, under this assumption, the cdf $F$ of $X_1$ satisfies the following: 
\begin{equation} \label{comment-ass}
	|F(u)-F(v)|\geq L|u-v|, \forall u,v\in [m-r,m+r].
\end{equation} 
Even though the existence of a density is not very restrictive in practice, it seems that our results would still be true if we only assumed the existence of a density in the neighborhood $[m-r,m+r]$ of $m$. Moreover, \eqref{comment-ass} is a natural and standard assumption on the distribution of $X_1$ in order to estimate its population median $m$ at the usual $n^{-1/2}$ rate. Indeed, if \eqref{comment-ass} does not hold, then the distribution of $X_1$ does not put enough mass around the median, which becomes harder to estimate. Moreover, it is well known that the empirical median of iid random variables is only asymptotically normal when the data have a positive density at the true median, and the asymptotic variance is $1/(4f(m)^2)$.

Recall that $\hat m(X)$ is the empirical median of $X=(X_1,\ldots,X_n)$. For simplicity, in the sequel, we write $\hat m$ instead of $\hat m(X)$. 

\begin{theorem}\label{thm:etaMedian}
Let $\DS n\geq \max\left(\frac{2\lceil C\rceil}{rL},\frac{2\log(8/\tau)}{(rL)^{2}}\right)$.
Then, for the empirical median, under Assumption \ref{Ass:1}, $\eta^*\leq \frac{4C}{Ln}+\frac{4\log(4/\tau)}{3Ln}$.
\end{theorem}

We can now control the error of the differentially private output $\tilde m$ of the Gaussian version of PTR, with $\eta$ given by the upper bound of Theorem \ref{thm:etaMedian}.

\begin{theorem} \label{thm:guaranteeMedian}
	Let Assumption \ref{Ass:1} hold, and let $\DS n\geq \max\left(\frac{2\lceil C\rceil}{rL},\frac{2\log(8/\tau)}{(rL)^{2}}\right)$. Then, the differentially private estimator $\tilde m$ can be computed in $O(n\log n)$ time and it satisfies, with probability $1-2\tau$,
$$|\tilde m-m|\leq \sqrt{\frac{\log(2/\tau)}{2nL^2}}+O\left(\frac{1}{\varepsilon^2 n}\right).$$
\end{theorem}

Note that the randomized estimator $\tilde m$ depends on the distributional parameters $r$ and $L$, which is limiting in practice. However, in a parametric setup, $r$ and $L$ can be known up to a scale parameter. For example, let the data be Gaussian with unknown mean $\mu$ and known variance $\sigma^2$. Then, $m=\mu$ and one can choose $r=\sqrt 2 \sigma$ and $L=(e\sqrt{2\pi\sigma^2})^{-1}$. This is a very simple example and yet, even for Gaussian data with known variance, no such result (high probability bound, yielding a non-asymptotic confidence interval, without any prior knowledge on the location of the population mean $\mu$) was previously known, to the best of our knowledge.

The leading term in Theorem \ref{thm:guaranteeMedian} is sub-Gaussian, i.e., it is of the order of $\sqrt{\frac{\log(1/\tau)}{n}}$, rescaled by $L^{-1}$, which is a scaling parameter as we have already seen above. The remaining term is of the order of $\frac{1}{\varepsilon^2 n}$: We do not know if this term is optimal. Indeed,  \cite[Section 3.2]{caietal2019}, proves optimal bounds in expectation for the estimation of the mean $\mu$ (equal to the median for symmetric distributions), under the assumption that $\mu$ is in some bounded domain; Their bounds are of the form $O(1/\sqrt n+1/(\varepsilon n))$, whereas the second term of our bound is $\varepsilon^{-1}$ times worse than theirs.
However, this may be the price to pay for high deviation bounds and for not allowing any prior knowledge on the location of the mean or the median. Bridging this gap is left for future work.

\subsection{Mean estimation}

In this section, we focus on the estimation of the mean of a distribution. It is easy to see that in general, for the empirical mean estimator, $\eta^*=\infty$, which is somewhat related to the lack of robustness of the empirical mean. When the distribution of the data is symmetric, the mean coincides with the median: It is estimated well by the empirical median which can be made differentially private without paying a significant price, as we saw in the previous section. Here, our idea is to first symmetrize the distribution, then compute an empirical median: the procedure that we adopt is the median of means (MOM) estimator \cite{nemirovskyandyudin1983, lerasleandoliveira2011, bubecketal2013}. The idea is is to first compute empirical means within separate blocks of data. The distribution of those empirical means is more symmetric than the initial distribution of the data, hence, its theoretical median is closer to its mean, which is what we are interested in. Hence, as a second step, estimate that median by computing the empirical median of all the empirical means. Let us define this estimator more formally.

Let $X_1,\ldots,X_n$ be our data, and let $K$ be some integer between $1$ and $n$. Let $N=\lfloor n/K\rfloor$ and split the data in $K$ disjoint blocks, each containing at least $N$ data points. In each block $j=1,\ldots,K$, compute the empirical mean $\bar X_j$ of the data points contained in block $k$. Then, the estimator $\hat \mu_{K}$ is defined as the empirical median of $\bar X_1,\ldots,\bar X_K$. The standard number of blocks prescribed in order to get a high probability guarantee for the estimation of the population mean $\mu$ of the data is of the order $\log(1/\tau)$, where $\tau\in (0,1)$ is the desired confidence level. However, in order to guarantee differential privacy through our PTR approach, we need some more freedom in the choice of $K$. We use the following result, due to \cite[Corollary 1]{minsker2019distributed}, that is based on the Berry-Esseen bound \cite{berry1941accuracy}:
\begin{lemma} \label{lemma:Minsker} 
	Let $X_1,\ldots,X_n$ be i.i.d. real random variables with three moments, and let $\mu=\E[X_1]$, $\sigma^2=\E[(X_1-\mu)^2]$ and $\rho^3=\E[|X_1-\mu|^3]$. Let $\tau\in (0,1)$. Provided that $ \frac{0.4748\rho^3}{\sigma^3\sqrt N}+\sqrt{\frac{\log(4/\tau)}{2K}}\leq 1/3$, it holds with probability at least $1-\tau$ that 
	$$|\hat \mu-\mu|\leq \sigma\left(\frac{1.43\rho^3}{\sigma^3 N}+3\sqrt{\frac{\log(4/\tau)}{2n}}\right).$$
\end{lemma}
We crucially leverage the above result in order to obtain the following bound on $\eta^*$ for the MOM estimator.
\begin{theorem} \label{thm:etaMOM}
	Suppose that $K\geq 8\max(C,\log(4/\tau))$. Then, $\eta^*\leq 2\sqrt 2\sigma\sqrt{K/n}$.
\end{theorem} 

Now, we set $\eta=2\sqrt 2\sigma\sqrt{K/n}$ and we prove the following high probability bound for the error of the product $\tilde \mu$ of the Gaussian version of PTR for the MOM estimator.

\begin{theorem} \label{thm:guaranteeMOM}
	Assume that the data have three finite moments: $\mu=\E[X_1]$, $\sigma^2=\E[(X_1-\mu)^2]$ and $\rho^3=\E[|X_1-\mu|^3]$. 
	Let $K\geq \max(4C,32\log(4/\tau))$ and assume that $n\geq 33(\rho/\sigma)^6K$. The randomized estimator $\tilde\mu$ can be computed in $O(n)$ time and, with probability at least $1-2\tau$,
$$|\tilde\mu-\mu|\leq \sigma\left(3\sqrt{\frac{\log(4/\tau)}{2n}}+\frac{4\sqrt{2K\log(2/\tau)\log(1.25/\delta)}}{\varepsilon \sqrt n} +\frac{1.43K\rho^3}{\sigma^3 n}\right).$$
\end{theorem}

Recall that for the differentially private estimation of the median, with confidence $\tau$, we obtained a bound of the form $ O\left(\sqrt{\frac{\log(1/\tau)}{n}}+\frac{poly\left(\log(1/\tau),\log(1/\delta)\right)}{\varepsilon^2 n}\right)$. Up to a $\varepsilon$ factor in the second term, this is a similar bound as the one obtained in expectation for the mean estimation in \cite[Section 3.2]{caietal2019}. Here, for our differentially private MOM estimator, the leading term (which is $O(1/\sqrt n)$) is no longer sub-Gaussian, since the price that we pay for differential privacy has the same order as the sub-Gaussian term, yielding a leading term that depends on the privacy parameters $\varepsilon$ and $\delta$. However, and perhaps surprisingly, when the data have a density, we significantly improve this bound.

\begin{theorem} \label{thm:guaranteeMOMdensity}

Let the data $X_1,\ldots,X_n$ be i.i.d. with three finite moments:	$\mu=\E[X_1]$, $\sigma^2=\E[(X_1-\mu)^2]$ and $\rho^3=\E[|X_1-\mu|^3]$. In addition, assume that $X_1$ has a density. 
Let $K\geq \max(8C,32\log(4/\tau))$ and assume that $n\geq 10(\rho/\sigma)^6K$ and that $n/K$ is an integer. 

\begin{enumerate}
	\item Then, $$\DS \eta^* \leq 2e^2\sigma \sqrt{2\pi}\left(\frac{\rho^3 K}{\sigma^2 n}+\frac{2C+(2/3)\log(4/\tau)}{\sqrt{Kn}}\right)=:\eta_0.$$
	\item Moreover, the randomized estimator $\tilde\mu$ with $\eta=\eta_0$ can be computed in $O(n)$ time and, with probability at least $1-2\tau$,
$$|\tilde\mu-\mu|\leq 3\sqrt{\frac{\sigma^2\log(4/\tau)}{2n}}+\frac{1.43\rho^3K}{\sigma^2n}  +\frac{4\sigma\sqrt{2\pi}}{e^{-2}\varepsilon}\left(\frac{\rho^3 K}{\sigma^3n}+\frac{2C+(2/3)\log(4/\tau)}{\sqrt{Kn}}\right).$$
\end{enumerate}
\end{theorem}

Note that the assumption that $n/K$ is an integer (which we did not need in Theorem \ref{thm:guaranteeMOM}, but which is needed in the proof of this theorem) is not restrictive, since otherwise we can drop some data without affecting the differential privacy and affecting the bound only up to universal constants.

In particular, taking $K$ of the order $n^{1/3}$, the error bound obtained in Theorem \ref{thm:guaranteeMOMdensity} is of the form $3\sqrt{\frac{\sigma^2\log(4/\tau)}{n}}+O\left(\frac{1}{\varepsilon^2 n^{2/3}}\right)$ which is, as desired, a sub-Gaussian term plus a negligible one.

\section{Conclusion}

We studied the problem of differentially private estimation of a location parameter from a non-asymptotic deviations perspective, by proposing a new Propose-Test-Release mechanism. The procedure first checks in a differentially private way, whether the data are in a favorable configuration for preserving privacy, in which case a carefully calibrated noisy version of the statistic of interest is released. We use our mechanism for the construction of differentially private median and mean estimators and we bound their finite sample performance with high probability. More precisely, these differentially private estimators exhibit leading sub-Gaussian error terms with high probability under minimal distributional assumptions needed to establish concentration results in the standard non-private setting. The  statistical properties of the new median estimator constitute a dramatic improvement over a previously proposed PTR median estimator. Our mean estimator is a PTR version of a newly proposed median-of-means estimator defined as the median of $K$ means obtained by averaging independent groups of $O(n^{2/3})$ observations. 

A salient practical advantage of our method over alternative approaches is that we avoid applying any type of truncation to the data and we do not need to assume that the target parameter lies in a known interval.  They seem to be the first algorithms capable of this. We view our results as a promising first step towards establishing a more general set of methods leading to optimal non-asymptotic concentration inequalities for differentially private estimators under mild distributional assumptions, and in a multivariate setting.

\section{Proofs}

\subsection{Proof of Theorem \ref{thm:DP_PTRLap}}

The proof relies on the sliding property of the Laplace distribution \cite{nissimetal2007}:

\begin{lemma} \label{Lemma:SlidingLap}
	Let $\varepsilon>0$. Let $Z$ be a Laplace random variable with parameter 1 and let $\Delta_0>0$. Then, for all $\Delta\in\R$ with $|\Delta|\leq\Delta_0$, and for all real Borel sets $\mathcal O$,
	$$\PP\left[\frac{\Delta_0}{\varepsilon}Z\in\mathcal O+\Delta\right]\leq e^\varepsilon \PP\left[\frac{\Delta_0}{\varepsilon}Z\in\mathcal O\right],$$
where $\mathcal O+\Delta=\{x+\Delta:x\in\mathcal O\}$.
\end{lemma}

First, note that $\tilde A_\eta$ is $(\varepsilon,0)$-differentially private. Indeed, if $\dist(x,x')\leq 1$, then $\hat A_\eta(x')$ can only take the values $\hat A_\eta(x)$ or $\hat A_\eta(x')\pm 1$. Hence, $\textsf{GS}_{\hat A_\eta}\leq 1$, and it follows from Lemma \ref{Lemma:SlidingLap} (see also, e.g., \cite{dworketal2006} on the Laplace mechanism) that $\tilde A_\eta$ is $(\varepsilon,0)$-differentially private.

Since  $\tilde{\theta}_\eta(x)=\perp \iff \tilde{A}_\eta(x)\leq 1+\frac{1}{\varepsilon}\log(2/\delta)$ and  $\tilde{A}_\eta$ is $(\varepsilon,0)$-differentially private, it follows that 
\begin{align} \label{lemma4:noreply} 
	\PP[\tilde \theta_\eta(x)=\perp] & = \PP\left[\tilde{A}_\eta(x)\leq 1+\frac{\log(\frac{2}{\delta})}{\varepsilon}\right] \\
	& \leq e^\varepsilon \PP\left[\tilde{A}_\eta(x')\leq 1+\frac{\log(\frac{2}{\delta})}{\varepsilon}\right] \\
	& = e^\varepsilon \PP[\tilde \theta_\eta(x')=\perp],
\end{align} 
for all $x,x'$ such that $\dist(x,x')\leq 1$. 

Let $x,x'\in\R^n$ with $\dist(x,x')\leq 1$. We will now show that for all Borel sets $\mathcal O\subseteq \R$,
\begin{equation}
\label{lemma4:reply}
\PP[\tilde \theta_\eta(x)\in\mathcal{O}]\leq  e^{2\varepsilon}\PP[\tilde\theta_\eta(x')\in \mathcal{O}]+\delta.
\end{equation}
Note that for $\tilde \theta_\eta(x)$ to be a real number, it has to be that the estimator has outputted a reply, i.e., $\tilde A_{\eta}(x)>1+(1/\varepsilon)\log(2/\delta)$. On the one hand, if  $|\hat\theta_\eta(x)-\hat\theta_\eta(x')|\leq \eta$, we have 
\begin{align}
\label{D0}
 &\PP[\tilde{\theta}_\eta(x)\in\mathcal{O}]\nonumber\\
 &=\PP\left[\hat{\theta}_\eta(x)+\frac{\eta}{\varepsilon}Z_2\in\mathcal{O},~\hat{A}(x)+\frac{1}{\varepsilon}Z_1> 1+\frac{\log(\frac{2}{\delta})}{\varepsilon}\right] \nonumber \\
 & =\PP\left[\hat{\theta}_\eta(x)+\frac{\eta}{\varepsilon}Z_2\in\mathcal{O}\right]\PP\left[\hat{A}(x)+\frac{1}{\varepsilon}Z_1 >1+\frac{\log(\frac{2}{\delta})}{\varepsilon}\right]\nonumber \\
  &\leq e^\varepsilon \PP\left[\hat{\theta}_\eta(x')+\frac{\eta}{\varepsilon}Z_2\in\mathcal{O}\right]e^\varepsilon\PP\left[\hat{A}(x')+\frac{1}{\varepsilon}Z_1 >1+\frac{\log(\frac{2}{\delta})}{\varepsilon}\right] \nonumber\\
 &=e^{2\varepsilon}\PP[\tilde{\theta}_\eta(x')\in \mathcal{O}]\nonumber \\
 &\leq e^{2\varepsilon}\PP[\tilde{\theta}_\eta(x')\in \mathcal{O}]+\delta,
\end{align}
by the sliding property of the Laplace distribution \cite[Section 2.1.1]{nissimetal2007} and where the second and the last equalities used independence of $Z_1$ and $Z_2$.  On the other hand, if $|\hat{\theta}_\eta(x)-\hat{\theta}_\eta(x')| > \eta$ then $\hat{A}_\eta(x)=\hat{A}_\eta(x')=1$ which in turn entails that
\begin{align}
\PP[\tilde{\theta}_\eta(x)\in\mathcal{O}]&=\PP\left[\hat{\theta}_\eta(x)+\frac{\eta}{\varepsilon}Z_2\in\mathcal{O},~\frac{1}{\varepsilon}Z_1> \frac{\log(\frac{2}{\delta})}{\varepsilon}\right] \nonumber \\
&\leq \PP\left[\frac{1}{\varepsilon}Z_1 >\frac{\log(\frac{2}{\delta})}{\varepsilon}\right]\nonumber\\
&=\delta \nonumber \\
&\leq e^{2\varepsilon}\PP[\tilde{\theta}_\eta(x')\in\mathcal{O}]+\delta \label{notD0}
\end{align}
Therefore, combining \eqref{D0} and \eqref{notD0} yields \eqref{lemma4:reply}. Now, let $\mathcal O'$ be a Borel set of the extended real line $\R\cup\{\perp\}$. Then, $\mathcal O'$ is equal to either $\mathcal O$, or $\mathcal O\cup\{\perp\}$, for some Borel set $\mathcal O$ of $\R$. In the former case, \eqref{lemma4:reply} concludes the proof of the lemma. In the latter case, we write, for all $x,x'\in\R^n$ with $\dist(x,x')\leq 1$:
\begin{align*}
	\PP[\tilde{\theta}_\eta(x)\in\mathcal{O'}] & = \PP[\tilde{\theta}_\eta(x)\in\mathcal{O}] + \PP[\tilde{\theta}_\eta(x)=\perp] \\
	& \leq e^{2\varepsilon}\PP[\tilde{\theta}_\eta(x')\in\mathcal{O}]+\delta + e^\varepsilon \PP[\tilde{\theta}_\eta(x')=\perp] \\
	& = e^{2\varepsilon} \PP[\tilde{\theta}_\eta(x')\in\mathcal{O'}]+\delta,
\end{align*}
thanks to \eqref{lemma4:noreply} and \eqref{lemma4:reply}.

\subsection{Proof of Theorem \ref{thm:DP_PTRGauss}}

The proof relies on the sliding property of the Gaussian distribution \cite[Definition 2.5]{nissimetal2007}:

\begin{lemma} \label{Lemma:SlidingGauss}
	Let $Z$ be a standard Gaussian random variable and let $\Delta_0>0$. Then, for all $\Delta\in\R$ with $|\Delta|\leq\Delta_0$, and for all real Borel sets $\mathcal O$,
	$$\PP\left[\frac{\Delta_0\sqrt{2\log(\frac{1.25}{\delta})}}{\varepsilon}Z\in\mathcal O+\Delta\right]\leq e^\varepsilon \PP\left[\frac{\Delta_0\sqrt{2\log(\frac{1.25}{\delta})}}{\varepsilon}Z\in\mathcal O\right]+\delta,$$
where $\mathcal O+\Delta=\{x+\Delta:x\in\mathcal O\}$.
\end{lemma}

The rest of the proof follows the same lines as the proof of Theorem \ref{thm:DP_PTRLap} and it is omitted here.

\subsection{Proof of Theorem \ref{thm:ErrorPTRLap}}

Of course, if $\eta^*=\infty$, the theorem is trivial, so let us assume that $\eta^*<\infty$ and let $\eta\geq\eta^*$. Since $\hat A_{\eta}\geq \hat A_{\eta^*}$, it is clear that 
$$\PP\left[\hat A_{\eta}(X)\leq 1+\frac{\log(1/(\tau\delta))}{\varepsilon}\right]\leq\tau/2.$$
Let $t=\log(1/\tau)$. Then, 
\begin{align*}
	\PP[\tilde\theta_{\eta}=\perp] & = \PP\left[\hat A_{\eta}+\frac{1}{\varepsilon}Z_1\leq 1+\frac{1}{\varepsilon}\log(1/\delta)\right] \\ 
	& = \PP\left[\hat A_{\eta}+\frac{1}{\varepsilon}Z_1\leq 1+\frac{1}{\varepsilon}\log(1/\delta), Z_1\geq -t\right]\\
	& \quad\quad\quad+\PP\left[\hat A_{\eta}+\frac{1}{\varepsilon}Z_1\leq 1+\frac{1}{\varepsilon}\log(1/\delta), Z_1< -t\right] \\
	& \leq \PP\left[\hat A_{\eta}\leq 1+\frac{1}{\varepsilon}\log(1/\delta)+\frac{t}{\varepsilon} \right] + \PP[Z_1< -t] \\
	& \leq \tau.
\end{align*}
Since, in addition, $Z_2\leq \log(1/(2\tau))$ with probability $1-\tau$, a union bound yields the desired result.

\subsection{Proof of Theorem \ref{thm:ErrorPTRGauss}}

The proof follows the exact same lines as the proof of Theorem \ref{thm:ErrorPTRLap}, with slight adaptations to the standard Gaussian distribution. It is omitted here.

\subsection{Proof of Theorem \ref{thm:etaMedian}}

Let $\eta>0$ and $k=\lceil C\rceil$, where $C=1+\frac{2\log(1.25/\delta)+2\sqrt{\log(2/\tau)\log(1.25/\delta)}}{\varepsilon}$. Consider the following events:
$$\mathcal A_k=\left\{m-r\leq X_{(\ell-k)}\leq X_{(\ell+k)}\leq m+r\right\}$$ 
and 
$$\mathcal B_{\eta,k}=\left\{X_{(\ell+k)}-X_{(\ell-k)}\leq \eta\right\}.$$ 
Recall that $\ell=\lfloor n/2\rfloor$ and, for simplicity, we assume that $n$ is even in the rest of the proof.

Let $X=(X_1,\ldots,X_n)$. It is clear that if $X'\in\R^n$ with $\dist(X,X')\leq k$, then $|\hat m(X)-\hat m(X')|\leq \max\left(X_{(\ell+k)}-X_{(\ell)},X_{(\ell)}-X_{(\ell-k)}\right)\leq X_{(\ell+k)}-X_{(\ell-k)}$.
Therefore, if $\mathcal B_{\eta,k}$ is satisfied, then one needs to change at least $k+1$ data points in order to move the empirical median $\hat m$ to a distance at least $\eta$, hence, $\hat A_\eta>k$. Therefore, we write 
\begin{align}
\PP[\hat A_{\eta}(X)\leq k] & \leq \PP[\mathcal B_{\eta,k}^{\complement}] \nonumber \\
& \leq \PP[\mathcal B_{\eta,k}^{\complement}\cap \mathcal A_k]+\PP[\mathcal A_k^{\complement}]. \label{eqn:etaMedianStep0}
\end{align} 
By a union bound,
\begin{equation} \label{eqn:etaMedianStep1}
	\PP[\mathcal A_k^{\complement}] \leq \PP[X_{(\ell-k)}<m-r]+\PP[X_{(\ell+k)}>m+r].
\end{equation}
Let us bound the first term; The second one will be bounded in a similar fashion. Note that  $X_{(\ell-k)}<m-r$ implies that there are at least $\ell-k$ data points that are less or equal to $m-r$. Thus, $\PP[X_{(\ell-k)}<m-r]$ can be bounded by $\PP[N\geq \ell-k]$, where $N$ is a binomial random variable with parameters $n$ and $p=\PP[X_1<m-r]\leq F(m-r)$, $F$ being the cdf of $X_1$. By Assumption \ref{Ass:1}, $F(m-r)\leq 1/2-Lr$ and Hoeffding's inequality yields
\begin{equation}
	\PP[X_{(\ell-k)}<m-r] \leq \exp\left(-2\left(\sqrt n Lr - \frac{k}{\sqrt n}\right)^2\right),
\end{equation}
since we have assumed that $n\geq k/(Lr)$.
Similarly, we have 
\begin{equation}
	\PP[X_{(\ell+k)}<m+r] \leq \exp\left(-2\left(\sqrt n Lr - \frac{k}{\sqrt n}\right)^2\right),
\end{equation}
and we get
\begin{equation} \label{eqn:etaMedianStep2}
\PP[\mathcal A_{k}^{\complement}]\leq 2\exp\left(-2\left(\sqrt n Lr - \frac{k}{\sqrt n}\right)^2\right).
\end{equation}
Now, let us bound the probability that $\mathcal B_{\eta,k}^{\complement}$ and $\mathcal A_k$ occur simultaneously. If this is the case, then Assumption \ref{Ass:1} implies that 
$$\eta\leq X_{(\ell+k)}-X_{(\ell-k)}\leq L^{-1}\left(F(X_{(\ell+k)})-F(X_{(\ell-k)})\right),$$ where we recall that $F$ is the cdf of $X_1$. Moreover, since we assume that the $X_i$'s are continuous random variables, if follows that $F(X_1),\ldots,F(X_n)$ are i.i.d. uniform random variables in $[0,1]$, and we can write
\begin{align} 
\PP[\mathcal B_{\eta,k}^{\complement}\cap \mathcal A_k] & \leq \PP[U_{(\ell+k)}-U_{(\ell-k)}\geq L\eta] \nonumber \\
& = \PP[U_{(2k)}\geq L\eta], \label{eqn:etaMedianStep3}
\end{align}
where we used that for all integers $s,t\in\{1,\ldots,n\}$ with $s< t$, $U_{(t)}-U_{(s)}$ has the same distribution as $U_{(t-s)}$. We can bound the right hand side of \eqref{eqn:etaMedianStep3} by the probability that a binomial random variable with parameters $n$ and $p=1-L\eta$ is greater or equal to $n-2k$, which yields, using Bernstein's inequality,
\begin{align}
	\PP[\mathcal B_{\eta,k}^{\complement}\cap \mathcal A_k] & \leq \exp\left(-\frac{3(Ln\eta-2k)^2}{4Ln\eta-2k}\right) \nonumber \\
	& \leq \exp\left(-(3/4)Ln\eta+6k)\right) \label{eqn:etaMedianStep4}
\end{align}
Our assumption on $n$ guarantees that $\PP[\mathcal A_k^{\complement}]\leq\tau/4$, and by taking $\DS \eta=\frac{4C}{Ln}+\frac{4\log(4/\tau)}{3Ln}$ guarantees that $\DS \PP[\mathcal B_{\eta,k}^{\complement}\cap\mathcal A_k]\leq\tau/4$, by \eqref{eqn:etaMedianStep4}. Hence, for this choice of $\eta$, we obtain, by \eqref{eqn:etaMedianStep0}, that $\DS \PP[\hat A_\eta(X)\leq C]\geq 1-\frac{\tau}{2}$,
i.e., $\eta^*\leq\eta$.

\subsection{Proof of Theorem \ref{thm:guaranteeMedian}}

For the first part, it suffices to show that $\hat{A}_\eta(x)$ can be computed in near linear time.  For this, we can sort $x$ and take the resulting order statistics $x_{(1)},\dots,x_{(n)}$ to compute $\hat{A}_\eta(x)=\min\{k \mbox{ s.t. } \max_{0\leq t\leq k+1} ( x_{(m+t)}-x_{(m+t-k-1)})>\eta\}$. For a fixed $k$, solving $\max_{0\leq t\leq k+1} ( x_{(m+t)}-x_{(m+t-k-1)})$ takes at most $O(n)$ operations. Furthermore, using a dichotomy method (which is valid since $\max_{0\leq t\leq k+1} ( x_{(m+t)}-x_{(m+t-k-1)})$ is monotone in $k$), we see that we only need to explore $O(\log(n))$ values of $k$ in order to find $\hat{A}_\eta(x)$. Hence for sorted $x$, we showed that $\hat{A}_\eta(x)$  can be computed in in $O(n\log (n))$ time. Since the initial sorting step also takes $O(n \log (n))$ operations, the overall algorithm is $O(n\log(n))$.

The second part of the theorem follows from combining Theorem \ref{thm:ErrorPTRGauss}, Theorem \ref{thm:etaMedian}, and the following fact, which we prove below.
\begin{lemma}
	Let Assumption \ref{Ass:1} hold and let $\tau\geq 2e^{-2nL^2r^2}$. Then, the empirical median $\hat m$ satisfies, with probability at least $1-\tau$,
$$|\hat m-m|\leq \sqrt{\frac{\log(2/\tau)}{2nL^2}}.$$
\end{lemma}
\begin{proof}
	The proof of this lemma is standard. Let $t=\sqrt{\frac{\log(2/\tau)}{2nL^2}}.$ Then, $\hat m\leq m-t$ implies that more than half of the data points are less or equal to $m-t$ and $\PP[\hat m\leq m-t]$ can be rewritten as $\PP[N\leq n/2]$, where $N$ is a binomial random variable with parameters $n$ and $p=F(m-t)$, where $F$ is the cdf of $X_1$. Thus, Hoeffding's inequality yields
$$\PP[\hat m\leq m-t] \leq e^{-2n(1/2-F(m-t))^2} \leq e^{-2nL^2t^2} \leq \frac{\tau}{2}, $$
where we have used the fact that $1/2-F(m-t)=F(m)-F(m-t)\geq Lt$, by Assumption \ref{Ass:1}.
Similarly, $\DS \PP[\hat m\geq m+t]\leq \tau/2$ and a union bound yields the desired result. 
\end{proof}

\subsection{Proof of Theorem \ref{thm:etaMOM}}

Let $k=\lceil C\rceil$. For all $\eta>0$, let $\DS\bar{\mathcal B}_{\eta,k}=\left\{\mu-\eta\leq \bar X_{(K/2-k)}\leq \bar X_{(K/2+k)}\leq\mu+\eta\right\}$ where $\bar X_{(1)}\leq \ldots\leq \bar X_{(K)}$ denote the order statistics of $\bar X_1,\ldots,\bar X_K$ and, for simplicity, we assume that $K$ is even. Then, it is clear that if $\mathcal B_{\eta,k}$ is satisfied, then $\hat A_\eta(X)\geq k$. Thus,
\begin{equation*}
	\PP[\hat A_\eta(X)\leq C]\leq \PP[\mathcal B_{\eta,k}].
\end{equation*}
Using the same reasoning as in the proof of Theorem \ref{thm:guaranteeMedian}, we bound the right-hand side by the probability that a binomial random variable with parameters $K$ and $p=\PP[\bar X_1\leq \mu-\eta]$ is greater or equal to $K/2-k$. By Chebychev's inequality, $p\leq \frac{K\sigma^2}{n\eta^2}$, yielding altogether
\begin{equation*}
	\PP[\hat A_\eta(X)\leq C]\leq 2 \exp\left(-2K\left(\frac{1}{2}-\frac{K\sigma^2}{n\eta^2}-\frac{k}{K}\right)^2\right).
\end{equation*}
The assumption on $K$ guarantees that by choosing $\eta=8\sigma\sqrt{K/n}$, we obtain 
\begin{equation*}
	\PP[\hat A_\eta\leq C]\leq \tau/2,
\end{equation*}
which ends the proof.

\subsection{Proof of Theorem \ref{thm:guaranteeMOM}}

First, the same reasoning as in the proof of Theorem \ref{thm:guaranteeMedian} yields that the computation of $\tilde \mu$ requires $O(K\log K)$ computations for $\hat A_\eta$, after computing all the means, which takes $O(n)$ operations. Overall, $\tilde\mu$ can be computed in $O(n)$ time.

To apply Lemma \ref{lemma:Minsker}, we need to check that $ \frac{0.4748\rho^3}{\sigma^3\sqrt N}+\sqrt{\frac{\log(4/\tau)}{K}}\leq 1/3$. Since the second term is, by assumption, no larger than $1/4$, it follows that $N$ needs to be at least $\DS \frac{144\cdot 0.4748^2\rho^6}{\sigma^6}$. It suffices that $n\geq 33(\rho/\sigma)^6K$. Then, Lemma \ref{lemma:Minsker} together with Theorem \ref{thm:etaMOM} yield the desired result.

\subsection{Proof of Theorem \ref{thm:guaranteeMOMdensity}}

The proof follows the same lines as in the proof of Theorem \ref{thm:etaMedian}, where assuming the existence of a density allowed us to control gaps between order statistics using i.i.d. uniform random variables. For simplicity, we assume that $K$ is even. Since $N=n/K$ is an integer, all the blocks contain the same number of data, hence, the empirical means $\bar X_1,\ldots,\bar X_K$ are i.i.d (they were only independent, not identically distributed before), which will be important in this proof.

Consider the events:
$$\bar{\mathcal A}_k=\left\{\mu-\frac{r\sigma}{\sqrt N}\leq \bar X_{(K/2-k)}\leq \bar X_{(K/2+k)}\leq \mu+\frac{r\sigma}{\sqrt N}\right\}$$
and 
$$\bar{\mathcal B}_{\eta,k}=\left\{\bar X_{(K/2+k)}-\bar X_{(K/2-k)}\leq\eta \right\},$$
where $r>0$ is some positive number.
As previously, it is clear that if $\mathcal B$ is satisfied, then $\hat A_\eta(X)\geq k$.

Then, following the same reasoning as previously, 
\begin{equation*}
	\PP[\bar{\mathcal A}_k^{\complement}]\leq \PP\left[\bar X_{(K/2-k)}<\mu-\frac{r\sigma}{\sqrt N}\right]+\PP\left[\bar X_{(K/2-k)}>\mu+\frac{r\sigma}{\sqrt N}\right].
\end{equation*}
Let us only bound the first term, since bounding the second term will be similar. The first term is bounded by the probability that a binomial random variable with parameters $K$ and $p=\PP[\bar X_{1}<\mu-\frac{r\sigma}{\sqrt N}]$ is larger or equal to $K/2-k$. By Chebychev's inequality, $p\leq 1/r^2$. Therefore, Hoeffding's inequality yields
\begin{equation*}
	\PP[\bar{\mathcal A}_k^{\complement}]\leq 2\exp\left(-(\sqrt K/2-\sqrt K/r^2-k/\sqrt K)^2\right).
\end{equation*}

Now, let us bound the probability that $\bar{\mathcal B}_{\eta,k}^{\complement}$ and $\bar{\mathcal A}_k$ occur simultaneously. Denote by $F_N$ the cumulative distribution function (cdf) of $\DS \sqrt N(\bar X_1-\mu)/\sigma$. By Berry-Esseen's inequality, for all $t\in \R$, 
$$|F_N(t)-\Phi(t)|\leq \frac{\rho^3}{\sigma^3\sqrt N},$$ 
where $\Phi$ is the standard Gaussian cdf. 
Moreover, for all $s,t\in [-r,r]$ with $t\geq s$, 
$$\Phi(t)-\Phi(s)\geq \frac{e^{-r^2/2}}{\sqrt{2\pi}}=:L_r.$$
Therefore, if $\bar{\mathcal B}_{\eta,k}^{\complement}$ and $\bar{\mathcal A}_k$ are both satisfied, it holds that 
$$F_N(Z_{(K/2+k)})-F_N(Z_{(K/2-k)})\geq \frac{L_r\sqrt N\eta}{\sigma} - \frac{2\rho^3}{\sigma^3\sqrt N},$$
where we denote by $\DS Z_i=\sqrt N\frac{\bar X_i-\mu}{\sigma}, i=1,\ldots,K$. Now, since the $X_i$'s have a density, the $Z_i$'s do as well, and the random variables $F_N(Z_i), i=1,\ldots,K$ are i.i.d. uniform in $[0,1]$. Hence, we can write that
$$\PP[\bar{\mathcal B}_{\eta,k}^{\complement}\cap\bar{\mathcal A}_k]\leq \PP\left[U_{(2k)}\geq \frac{L_r\sqrt N \eta}{\sigma}-\frac{2\rho^3}{\sigma^3\sqrt N}\right],$$
where $U_{(2k)}$ is the $(2k)$-th order statistic of a sample of $K$ i.i.d. uniform random variables in $[0,1]$.
Therefore, following the same lines as in the proof of Theorem \ref{thm:etaMedian}, and using again Bernstein's inequality, we obtain
\begin{equation*}
	\PP[\bar{\mathcal B}_{\eta,k}^{\complement}\cap\bar{\mathcal A}_k]\leq \exp\left(-(3/4)Kt+3k\right),
\end{equation*}
where $t=\frac{L_r\sqrt N\eta}{\sigma} - \frac{2\rho^3}{\sigma^3\sqrt N}$.
Finally, set $r=2$. Since $K\geq \max(8C, 32\log(4/\tau))$, $\PP[\bar{\mathcal A}_k^{\complement}]\leq\tau/4$ and taking $\DS \eta=2e^2\sigma \sqrt{2\pi}\left(\frac{\rho^3}{\sigma^2 N}+\frac{2C+(2/3)\log(4/\tau)}{K\sqrt N}\right)$, we obtain that
\begin{align*}
	\PP[\hat A_\eta\leq C] & \leq \PP[\bar{\mathcal B}_{\eta,k}^{\complement}] \\ 
	& \leq \PP[\bar{\mathcal B}_{\eta,k}^{\complement}\cap\bar{\mathcal A}_k]+\PP[\bar{\mathcal A}_k^{\complement}] \\
	& \leq \frac{\tau}{4} +\frac{\tau}{4} = \frac{\tau}{2}.
\end{align*}
This shows the first part of the theorem.

The rest of the proof is the same as that of Theorem \ref{thm:guaranteeMOM}, by combining Lemma \ref{lemma:Minsker} and Theorem \ref{thm:ErrorPTRGauss}.

\bibliographystyle{plain}
\bibliography{Biblio}

\end{document}